%% file: main.tex
\documentclass[11pt,a4paper,final]{article}
\usepackage[left=1in,right=1in,top=1in,bottom=1in]{geometry}

\usepackage[numbers]{natbib}

\usepackage[utf8]{inputenc} 
\usepackage[T1]{fontenc}    
\usepackage{hyperref}       
\usepackage{url}            
\usepackage{booktabs}       
\usepackage{amsfonts}       
\usepackage{nicefrac}       
\usepackage{microtype}      
\usepackage{xcolor}         

\usepackage[ruled,vlined]{algorithm2e}
\usepackage[utf8]{inputenc}

\usepackage{amsmath,amssymb,amsthm}

\newtheorem{Thm}{Theorem}[section]
\newtheorem{Lem}{Lemma}[section]
\newtheorem{Cor}{Corollary}[section]
\newtheorem{Clm}{Claim}[section]

\DeclareMathOperator*{\argmin}{arg\,min}

\newcommand{\BigO}{\mathcal O}

\newcommand{\Breg}{B}
\newcommand{\diag}{\operatorname{diag}}
\newcommand{\OMD}{\texttt{OMD}}
\newcommand{\OPT}{\operatorname{OPT}}

\newcommand{\C}{\mathcal C}

\newcommand{\E}{\mathbb E}

\newcommand{\K}{\mathcal K}

\newcommand{\R}{\mathbb R}

\usepackage{scalerel,mathtools,color}
\let\svsqrt\sqrt
\newsavebox\Nsqrt
\def\sr#1{\ThisStyle{%
		\savebox\Nsqrt{\scalebox{.5}[1]{$\SavedStyle\svsqrt{\phantom{\cramped{#1#1}}}$}}%
		\ooalign{\usebox{\Nsqrt}\cr\kern.2pt\usebox{\Nsqrt}\cr\hfil$\SavedStyle\cramped{#1}$}}}

\usepackage{enumitem}
\usepackage{bm}
\def\*#1{\mathbf{#1}}

\usepackage{threeparttable}

\title{Meta-Learning Adversarial Bandits
}

\author{%
  Maria-Florina Balcan\qquad Keegan Harris\qquad Mikhail Khodak\qquad Zhiwei Steven Wu\\\\
  Carnegie Mellon University School of Computer Science\\\\
  \texttt{\{ninamf,keeganh\}@cs.cmu.edu,\{khodak,zstevenwu\}@cmu.edu}\\
}

\begin{document}

\date{}
\maketitle

\begin{abstract}
  We study online learning with bandit feedback across multiple tasks, with the goal of improving average performance across tasks if they are similar according to some natural task-similarity measure.
  As the first to target the adversarial setting, we design a unified meta-algorithm that yields setting-specific guarantees for two important cases:
  multi-armed bandits (MAB) and bandit linear optimization (BLO).
  For MAB, the meta-algorithm tunes the initialization, step-size, and entropy parameter of the Tsallis-entropy generalization of the well-known Exp3 method, with the task-averaged regret provably improving if the entropy of the distribution over estimated optima-in-hindsight is small.
  For BLO, we learn the initialization, step-size, and boundary-offset of online mirror descent (OMD) with self-concordant barrier regularizers, showing that task-averaged regret varies directly with a measure induced by these functions on the interior of the action space.
  Our adaptive guarantees rely on proving that unregularized follow-the-leader combined with multiplicative weights is enough to online learn a non-smooth and non-convex sequence of affine functions of Bregman divergences that upper-bound the regret of OMD.\looseness-1
\end{abstract}

%
%

\input{intro}
\input{related}
\input{setup}
\input{mab}
\input{blo}
\input{conclusion}

\section*{Acknowledgments}
This material is based on work supported by the National Science Foundation under grants CCF-1910321, FAI-1939606, IIS-1901403, SCC-1952085, and SES-1919453; the Defense Advanced Research Projects Agency under cooperative agreement HR00112020003; a Simons Investigator Award; an AWS Machine Learning Research Award; an Amazon Research Award; a Bloomberg Research Grant;  a Microsoft Research Faculty Fellowship; a Google Faculty Research Award; a J.P. Morgan Faculty Award; a Facebook Research Award; a Mozilla Research Grant; and a Facebook PhD Fellowship.

\medskip

\bibliographystyle{plainnat}
\bibliography{bandit-refs,refs}



\appendix

\input{proofs}

\end{document}

%% file: intro.tex

\section{Introduction}\label{sec:intro}
Many real-world problems involve decision-making under partial feedback. For example, an administrator of a news website will not get to observe user engagement for undisplayed content. Likewise, the administrator does not know what their commute time would have been had they taken a different route to work that day. Such  feedback, where a \emph{learner} only observes the outcome of the action taken, is often referred to as \emph{bandit feedback}. This is in contrast to the \emph{full feedback} setting, in which the learner gets to observe what would have happened under all possible actions they could have taken. 

While there are many methods with performance guarantees for bandit learning, most do not take into consideration the information the learner has gained from previous experience completing similar tasks. When selecting content for a website targeting a new demographic, the administrator will likely consider which types of content generated high levels of engagement with similar subpopulations of users. Likewise, they will most likely use knowledge about traffic patterns gained from their daily commute to work to inform routes to other locations. 
Can our learning algorithms do the same?

\emph{Meta-learning}, also known as \emph{learning-to-learn} \cite{thrun1998ltl}, is a popular approach to studying such problems in the context of multi-task learning, changing environments, and beyond. 
The goal of such meta-learning algorithms is to leverage information from previously-seen tasks in order to achieve better performance on the current task at hand. While most meta-learning algorithms are designed for the full feedback setting, there is a small but growing amount of recent work which aims to design meta-learning algorithms capable of operating under bandit feedback. 
The two most commonly studied types of feedback in the bandit literature are \emph{stochastic bandit feedback}, where feedback is sampled i.i.d. from some distribution, and a more general notion called \emph{adversarial bandit feedback}, where it is chosen by an adversary possibly trying to harm the learner. To the best of our knowledge, we are the first to study the problem of meta-learning under adversarial bandit feedback.

We consider a setting in which a meta-learner interacts with a sequence of bandit tasks.
In the single-task setting, the goal is to minimize \emph{regret} with respect to the best fixed action in hindsight.
Lifting this to the multi-task setting, our goal will be to design algorithms which achieve low regret \emph{on average} across tasks. Ideally, an algorithm's task-averaged regret should be no worse than that of algorithms for the single-task setting, e.g. if the tasks are not very similar, and should be much better on tasks that are closely related, e.g. if the same small set of arms do well on all of them.

We design a meta-algorithm based on learning the initialization and tuning parameters of online mirror descent (OMD) when it uses regularizers employed by bandit algorithms such as Exp3 \citep{auer2002exp3}.
Theoretically, our work differs from past approaches to parameterizing OMD in the full information setting because the regularizers used by bandit methods are non-Lipschitz near the boundary of the action set;
thus the results of past work~\citep{khodak2019provable,khodak2019adaptive,denevi2019meta} do not apply.
To overcome this issue, we only initialize away from the boundary and adapt our algorithms to handle the resulting error.

We apply our meta-algorithm for adversarial bandit feedback to both multi-armed bandits (MAB) and bandit linear optimization (BLO), obtaining in both settings new meta-learning algorithms with provable guarantees. For MAB, the average $m$-round regret across $T$ tasks of our algorithm is
\begin{equation}\label{eq:result}
	o_T(1)+2\min_{\beta\in(0,1]}\sqrt{\hat H_\beta d^\beta m/\beta}
\end{equation}
where $d$ is the number of actions and $\hat H_\beta$ is the Tsallis entropy~\citep{tsallis1988possible,abernethy2015fighting} of the empirical distribution over the estimated optimal actions across tasks.
At $\beta=1$ the Tsallis entropy reduces to the Shannon entropy;
both are small if most tasks are estimated to be solved by the same few arms and large if all are used roughly the same amount, making it a natural task-similarity notion.
The bound of $\log d\ge\hat H_1$ means that the bound~\eqref{eq:result} recovers Exp3's guarantee in the worst-case of dissimilar tasks.
In the important case of $s\ll d$ arms always being estimated to be optimal we have $\hat H_\beta=\BigO(s)$, so using $\beta=\frac1{\log d}$ in bound~\eqref{eq:result} yields a task-averaged regret of $\BigO(\sqrt{sm\log d})$ as $T\to\infty$.
For $s=\BigO_d(1)$ this beats the single-task lower bound of $\Omega(\sqrt{dm})$~\citep{audibert2011minimax}.
We also obtain natural task-averaged regret bounds for BLO, albeit with different setting-specific notions of task similarity. 

Our main technical contributions are as follows:
\begin{enumerate}[noitemsep]
    \item We design a unified meta-learning algorithm to set the initialization and tuning parameters of OMD when using regularizers used by different bandit algorithms (Algorithm~\ref{alg:meta}).
    Apart from strong guarantees and generality, our approach is notable for its adaptivity:
    we do not need to know anything about the task-similarity---e.g. the size of the subset of optimal arms---to adapt to similar tasks.\looseness-1
    \item We apply our meta-approach to obtain a meta-learning algorithm for the adversarial MAB problem. In particular, we use the method of \citet{abernethy2015fighting}---OMD with the Tsallis regularizer---as our within-task algorithm to achieve bounds on task-averaged regret that depend on a natural notion of task similarity:
    the Tsallis entropy of the estimated optima-in-hindsight.
    \item We adapt Algorithm \ref{alg:meta} to the adversarial BLO problem by setting the regularizer to be a self-concordant barrier function, as in~\citet{abernethy2008competing}. As in MAB, we obtain task-averaged regret bounds which depend on a natural notion of task similarity based on the constraints defining the convex action space.
    We instantiate the BLO result in two settings:
    linear bandits over the sphere and an application to the bandit shortest-path problem~\citep{takimoto2003path,kalai2005efficient}.
\end{enumerate}

%% file: related.tex

\section{Related work}

While we are the first to consider meta-learning under adversarial bandit feedback, many have studied meta-learning in various {\em stochastic} bandit settings~\citep{sharaf2021meta, simchowitz2021bayesian, kveton2020meta, cella2020meta, kveton2021meta, basu2021no, azizi2022non, lazaric2013sequential}. \citet{kveton2021meta}, \citet{basu2021no}, and \citet{simchowitz2021bayesian} study meta-learning algorithms for the Bayesian bandit setting. \citet{kveton2020meta} and \citet{sharaf2021meta} consider meta-learning for contextual bandits, although they allow their algorithms to have offline access to a set of training tasks for which full feedback is available. \citet{cella2020meta} and \citet{moradipari2022multi} provide algorithms based on OFUL~\citep{abbasi2011improved} for meta-learning in stochastic linear bandits under various assumptions on how the bandit learning tasks are generated. \citet{azizi2022non} study a setting in which a meta-learner faces a sequence of stochastic multi-armed bandit tasks. While the sequence of tasks may be adversarially designed, the adversary is constrained to choose the optimal arm for each task from a smaller but unknown subset of arms. In contrast to \cite{cella2020meta, moradipari2022multi, azizi2022non}, we make no assumptions about how the sequence of tasks is generated and our guarantees adapt to a natural measure of similarity between tasks.

Theoretically our analysis draws on the average regret-upper-bound analysis (ARUBA) framework of \citet{khodak2019adaptive}, which was designed for meta-learning under full information.
While the general approach is not restricted by convexity~\citep{balcan2021ltl} and has  been combined with bandit algorithms on the meta-level~\citep{khodak2021fedex}, the existing results cannot be applied to OMD methods for within-task learning under bandit feedback because the associated regularizers are non-Lipschitz or sometimes even unbounded near the boundaries of the action space.
We thus require a specialized analysis for the bandit setting.
\citet{denevi2019meta} also study an OMD-based algorithm for meta-learning in the online setting, but their results are also only applicable in the full information setting.\looseness-1

%% file: setup.tex

\section{Learning the regularizers of bandit algorithms}\label{sec:setup}

We consider the problem of meta-learning across bandit tasks $t=1,\dots,T$ over some fixed set $\K\subset\R^d$. 
On each round $i=1,\dots,m$ of task $t$ we play action $\*x_{t,i}\in\K$ and receive feedback $\ell_{t,i}(\*x_{t,i})$ for some function $\ell_{t,i}:\K\mapsto[-1,1]$.
Note that all functions we consider will be linear and so we will also write $\ell_{t,i}(\*x)=\langle\ell_{t,i},\*x\rangle$. Additionally, we allow each $\ell_{t,i}$ to be chosen by an \emph{oblivious adversary}, i.e. an adversary with knowledge of the algorithm that must select $\ell_{t,i}$ independent of $\*x_{t,i}$.
We will also denote $\*x(a)$ to be the $a$th element of the vector $\*x\in\R^d$, $\overline\K$ to be the convex hull of $\K$, and $\triangle_n$ to be the simplex on $n$ elements.
Finally, note that all proofs can be found in the Appendix.\looseness-1

In online learning, the goal on a single task $t$ is to play actions $\*x_{t,1},\dots\*x_{t,m}$ that minimize the regret $\sum_{i=1}^m\ell_{t,i}(\*x_{t,i})-\ell_{t,i}(\*x_t^\ast)$, where $\*x_t^\ast\in\argmin_{\*x\in\K}\sum_{i=1}^m\ell_{t,i}(\*x)$.
Lifting this to the meta-learning setting, our goal as in past work \citep{khodak2019adaptive,balcan2021ltl} will be to minimize the {\bf task-averaged regret}
\begin{equation}\label{eq:tar}
	\frac1T\sum_{t=1}^T\sum_{i=1}^m\ell_{t,i}(\*x_{i,t})-\ell_{t,i}(\*x_t^\ast)
\end{equation}
In-particular, we hope to use multi-task data in order to improve average performance as the number of tasks $T\to\infty$, e.g. by attaining a task-averaged regret of $o_T(1)+\tilde\BigO(V\sqrt m)$, where $V\in\R_{\ge0}$ is a measure of task-similarity that is small if the tasks are similar but still yields the worst-case single-task performance if they are not.

\subsection{Online mirror descent as a base-learner}\label{ssec:mirror}

In meta-learning we are commonly interested in learning a within-task algorithm or {\bf base-learner}, a parameterized method that we run on each task $t$.
A popular approach, both empirically~\citep{finn2017maml,nichol2018reptile} and theoretically~\citep{khodak2019adaptive,denevi2019ltlsgd}, is to learn the initialization and sometimes other parameters of a gradient-based method such as stochastic gradient descent.
The hope is that optimal parameters for each task are close to each other and thus a meta-learned initialization will result in a strong model after only a few steps.
In this paper we take a similar approach applied to online mirror descent, a generalization of gradient descent to non-Euclidean geometries~\citep{beck2003mirror}.
Given a strictly convex {\bf regularizer} $\phi:\overline\K\mapsto\R$ and step-size $\eta>0$, this method performs the update\looseness-1
\begin{equation}\label{eq:mirror}
	\*x_{t,i+1}=\argmin_{\*x\in\overline\K}\Breg_\phi(\*x||\*x_{t,1})+\eta\sum_{j<i}\langle\nabla\ell_{t,j}(\*x_{t,j}),\*x\rangle
\end{equation}
where $\Breg_\phi(\*x||\*y)=\phi(\*x)-\phi(\*y)-\langle\nabla\phi(\*y),\*x-\*y\rangle$ is the {\bf Bregman divergence} of $\phi$.
OMD recovers online gradient descent when $\phi(\*x)=\frac12\|\*x\|_2^2$, in which case $\Breg_\phi(\*x||\*y)=\frac12\|\*x-\*y\|_2^2$;
another important example is {\bf exponentiated gradient}, for which $\phi(\*p)=\langle\*p,\log\*p\rangle$ is the negative Shannon entropy on probability vectors $\*p\in\triangle_n$ and $\Breg_\phi$ is the KL-divergence~\citep{shalev-shwartz2011oco}.
An important property of $\Breg_\phi$ is that the sum over functions $\Breg_\phi(\*x_t||\cdot)$ is minimized at the mean $\bar{\*x}$ of the points $\*x_1,\dots,\*x_T$.

While originally developed for online convex optimization, mirror descent using {\bf loss estimators} $\hat\ell_{t,i}$ constructed using bandit feedback $\ell_{t,i}(\*x_{t,i})$ forms an important class of methods for bandit settings~\citep{abernethy2008competing,neu2015explore,abernethy2015fighting}, including the famous Exp3 method \citep{auer2002exp3}.
Learning the initialization for online mirror descent has been considered for full-information meta-learning \citep{khodak2019adaptive,denevi2019meta}, but these papers do not apply to the types of regularizers $\phi$ required for bandits, which are often non-Lipschitz and sometimes even unbounded on the boundary of $\overline\K$ \citep{abernethy2008competing,abernethy2015fighting}.
For example, \citet{khodak2019adaptive} also take advantage of the mean-as-minimizer property of $\Breg_\phi$ and learn both the initialization and step-size $\eta$, but they assume the gradient of $\Breg_\phi$ is bounded on the domain, which does not hold if $\phi$ is non-Lipschitz, e.g. if $\phi$ is the negative Shannon entropy as in Exp3.

In this paper we resolve these issues by meta-learning to initialize and tune mirror descent when it employs a regularizer used by bandit methods.
Following the average regret-upper-bound analysis (ARUBA) framework of \citet{khodak2019adaptive}, we do this by online learning a sequence of losses $U_t(\*x,\theta)$, each of which is a hyperparameter $\theta$-dependent affine function of a Bregman divergence from an initialization $\*x\in\overline\K$ to some known fixed point in $\overline\K$.
We are interested in learning such functions because the regret after $m$ rounds of OMD initialized at $\*x$ with step-size $\eta$ is usually upper-bounded by $\frac1\eta\Breg_\phi(\*x_t^\ast||\*x)+\BigO(\eta m)$ for $\*x_t^\ast$, the optimum-in-hindsight on task~$t$~\citep{shalev-shwartz2011oco,hazan2015oco}.\looseness-1

Unlike past work, we use a parameter $\varepsilon>0$ to constrain this optimum to lie in a convex subset $\K_\varepsilon\subset\overline\K$ whose boundary is $\varepsilon$-away from that of $\overline\K$ and which satisfies $\K_\varepsilon\subset\K_{\varepsilon'}$ whenever $\varepsilon\le\varepsilon'$;
for example, we use $\K_\varepsilon=\{\*x\in\triangle_d:\min_a\*x(a)\ge\varepsilon/d\}$ for the simplex.
Thus, unlike with full-information, the feedback we receive from the within-task algorithm will be the minimizer $\OPT_\varepsilon(\hat\ell_t)=\argmin_{\*x\in\K_\varepsilon}\langle\hat\ell_t,\*x\rangle$ of the estimated loss $\hat\ell_t=\sum_{i=1}^m\hat\ell_{t,i}$ over the $\varepsilon$-constrained subset, where we can pick $\varepsilon\in(0,1)$.
This allows us to handle regularizers that diverge near the boundary, but also introduces $\varepsilon$-dependent error terms to handle.
In the BLO case it also forces us to automatically tune $\varepsilon$ itself, as initializing too close to the boundary leads to unbounded regret while initializing too far away does not take advantage of the task-similarity.

Thus in full generality the upper bounds of interest are functions of the initialization $\*x$ and three parameters:
the step-size $\eta>0$, a parameter $\beta$ of the regularizer $\phi_\beta$, and the boundary offset $\varepsilon>0$.
\begin{equation}\label{eq:rub}
	U_t(\*x,(\eta,\beta,\varepsilon))
	=\frac{\Breg_{\phi_\beta}(\OPT_\varepsilon(\hat\ell_t)||\*x)}\eta+(\eta G_\beta^2+C\varepsilon)m
\end{equation}
Here $G_\beta\ge1,C\ge0$ are constants and $\beta$ parameterizes the regularizer $\phi_\beta$, e.g. the negative Tsallis entropy used to attain optimal dependence on dimension for MAB~\citep{abernethy2015fighting}.
The reason to optimize this sequence of upper bounds is because the resulting average regret directly bounds the task-averaged regret, apart from some $o_T(1)$ terms.
Furthermore, an affine sum over Bregman divergences is minimized at the average optimum in hindsight, which leads to natural and problem specific task-similarity measures $V$ \citep{khodak2019adaptive};
specifically, $V$ is the square root of the average divergence between optima in hindsight and their mean, which is small if the tasks are optimized by similar parameters.\looseness-1

\subsection{A meta-algorithm for tuning bandit algorithms}\label{ssec:algo}

Having specified our meta-goal---learning to initialize and tune OMD for regularizers $\phi_\beta$ use in bandit tasks---we now detail our meta-algorithm for doing so, pseudo-code for which is in Algorithm~\ref{alg:meta}.
At a high level, the method simultaneously learns the initialization by taking the mean of $\K_\varepsilon$-constrained estimated optima-in-hindsight---i.e. follow-the-leader over the Bregman divergences in \eqref{eq:rub}---while simultaneously tuning OMD via multiplicative weights over a discrete grid $\Theta$ over $\theta=(\eta,\beta,\varepsilon)$.

In more detail, the algorithm assumes two primitives discussed above: 
(1) the base-learner $\OMD_{\eta,\beta}$ that outputs an estimated cumulative loss $\hat\ell_t\in\R^d$ after running online mirror descent over the $m$ losses $\ell_{t,1},\dots,\ell_{t,m}$ of task $t$, and (2) an optimizer $\OPT_\varepsilon$ that, given a vector $\*c\in\R^d$, finds the minimum of $\langle\*c,\cdot\rangle$ over $\K_\varepsilon$.
Algorithm~\ref{alg:meta} maintains a categorical distribution $\*p_t$ over a finite set $\Theta\subset\R^3$ containing triples $\theta=(\eta,\beta,\varepsilon)$, each with its own associated initialization $\*x_t^{(\theta)}$;
at each task $t$ it samples $\theta_t=(\eta_t,\beta_t,\varepsilon_t)$ from $\Theta$ using $\*p_t$ and runs $\OMD_{\eta_t,\beta_t}$ from initialization $\*x_t^{(\theta_t)}$, obtaining a loss estimate $\hat\ell_t$.
Then for each $\theta=(\eta,\beta,\varepsilon)$ in $\Theta$ the method updates the corresponding initialization $\*x_t^{(\theta)}$ by taking the average of the $\varepsilon$-constrained optima-in-hindsight $\OPT_\varepsilon(\hat\ell_1),\dots,\OPT_\varepsilon(\hat\ell_t)$ seen so far.
Finally, the algorithm updates the distribution $\*p_t$ using multiplicative weights over the following modification of the regret-upper-bound \eqref{eq:rub} above for some $\rho>0$:
\begin{align}\label{eq:modrub}
U_t^{(\rho)}(\*x,\theta)=\frac{\Breg_{\phi_\beta}(\hat{\*x}_t^{(\theta)}||\*x)+\rho^2D^2}\eta+(\eta G_\beta^2+C\varepsilon)m
\end{align}
Note that given $\rho>0$ this function is fully defined after running $\OMD_{\eta_t,\beta_t}$ on task $t$ to obtain loss estimates $\hat\ell_t$ and then computing the $\varepsilon$-constrained optimum-in-hindsight $\hat{\*x}_t^{(\theta)}=\OPT_\varepsilon(\hat\ell_t)$ for each $\theta=(\eta,\beta,\varepsilon)$.
This allows us to use full-information multiplicative weights for $\theta$.
$\rho>0$ is necessary for learning $\eta$, as if its optimum is near zero then $U_t$ will not be Lipschitz near the optimum.
Theorem~\ref{thm:meta} shows a sublinear regret guarantee for Algorithm~\ref{alg:meta} over the unmodified regret-upper-bounds \eqref{eq:modrub} w.r.t. all elements in $\overline\K$ and in a continous set of hyperparameters $\Theta^\ast\subset\R^3$.

\begin{algorithm}[!t]
	\DontPrintSemicolon
	\KwIn{compact $\overline\K\subset\R^d$, 
			meta-hyperparameters $\alpha,\rho>0$,
			finite $\Theta\subset\R^3$ over $(\eta,\beta,\varepsilon)$,
			base-learner $\OMD_{\eta,\beta}:\overline\K\mapsto\R^d$, 
			constrained linear minimizer $\OPT_\varepsilon:\R^d\mapsto\K_\varepsilon$}
	\For{$\theta=(\eta,\beta,\varepsilon)\in\Theta$}{
		$\*x_1^{(\theta)}\gets\argmin_{\*x\in\overline\K}\phi(\*x)$\tcp*{maintain an initialization for each $\theta\in\Theta$}
	}
	$\*p_1\gets\*1_{|\Theta|}/|\Theta|$\tcp*{multiplicative weights (MW) initialization}
	\For{task $t=1,\dots,T$}{
		sample $\theta_t=(\eta_t,\beta_t,\varepsilon_t)\sim\*p_t$ from $\Theta$\\
		$\hat\ell_t\gets\OMD_{\eta_t,\beta_t}(\*x^{(\theta_t)})$\tcp*{run bandit OMD within-task}
		\For{$\theta=(\eta,\beta,\varepsilon)\in\Theta$}{
			$\*x_{t+1}^{(\theta)}\gets\frac1t\sum_{s=1}^t\OPT_\varepsilon(\hat\ell_s)$\tcp*{update all initializations}
			$\*p_{t+1}(\theta)\gets\*p_{t+1}(\theta)\exp\left(-\alpha U_t^{(\rho)}(\*x_t^{(\theta)},\theta)\right)$\tcp*{MW update using loss in \eqref{eq:modrub}}
		}
		$\*p_{t+1}\gets\*p_{t+1}/\|\*p_{t+1}\|_1$
	}
	\caption{\label{alg:meta}
		Algorithm for tuning an online mirror descent (OMD) base-learner $\OMD_{\eta,\beta}$ with parameterized regularizer $\phi_\beta:\overline\K\mapsto\R$ and step-size $\eta>0$ that runs OMD on loss estimators $\hat\ell_{t,1},\dots,\hat\ell_{t,m}$ from an initialization $\*x\in\overline\K$ and returns estimated loss $\hat\ell_t=\sum_{i=1}^m\hat\ell_{t,i}\in\R^d$.
		Then for every $\varepsilon>0$ the constrained optimizer $\OPT_\varepsilon(\hat\ell)=\argmin_{\*x\in\K_\varepsilon}\langle\hat\ell,\*x\rangle$ returns the minimizer of the estimated loss over the constrained subset $\K_\varepsilon\subset\overline\K$ (set $\OPT_\varepsilon(\*0_d)=\argmin_{\*x\in\overline\K}\phi(\*x)$).
	}
\end{algorithm}

\begin{Thm}\label{thm:meta}
	Let $\Theta^\ast=(0,\infty)\times[\underline\beta,\overline\beta]\times[\underline\varepsilon,\overline\varepsilon]$ for $0\le\underline\beta\le\overline\beta\le1$ and $0\le\underline\varepsilon\le\overline\varepsilon\le1$ be the set of hyperparameters $(\eta,\beta,\varepsilon)$ of interest.
	Then there exists integer $k=\BigO(\lceil\sqrt{mT}\rceil)$ and $\alpha,\underline\eta,\overline\eta\in(0,\infty)$ such that running Algorithm~\ref{alg:meta} with $\Theta$ the product of uniform grids of size $k$ over each non-singleton dimension of  $[\underline\eta,\overline\eta]\times[\underline\beta,\overline\beta]\times[\underline\varepsilon,\overline\varepsilon]$ and $\alpha$ the meta-step-size yields regret
	\begin{align}\label{eq:metareg}
	\begin{split}
	\E&\sum_{t=1}^TU_t(\*x_t^{(\theta_t)},\theta_t)-\min_{\*x\in\overline\K,\theta\in\Theta^\ast}\sum_{t=1}^TU_t(\*x,\theta)
	\\
	&\le\left(C\sqrt m+2DG\left(\frac1\rho+M\right)\right)\sqrt{6mT\log k}+\frac{8SK^2G\sqrt m}{\rho D}(1+\log T)+\rho DGT\sqrt m
	\end{split}
	\end{align}
	for $G=\max_\beta G_\beta\ge1$, 
	$M=\frac G{\min_\beta G_\beta}$,  $D^2=\max_{\beta,\varepsilon,\*x,\*y\in\K_\varepsilon}\Breg_{\phi_\beta}(\*x||\*y)\ge1$, 
	$L$ the maximum Lipschitz constant of $\phi_\beta(\OPT_\varepsilon(\ell))$ w.r.t. $(\beta,\varepsilon)$ over $\ell\in\R^d$, $S=\max_{\beta,\varepsilon,\*x\in\K_\varepsilon}\|\nabla^2\phi_\beta(\*x)\|_2$, 
	$K=\max_{\*x,\*y\in\K}\|\*x-\*y\|_2$, 
	and the expectation is over sampling $\theta_t\sim\*p_t$.
	The result without the expectation holds w.p. $1-\delta$ at the cost of an additional $\left(C\sqrt m+2DG\left(\frac1\rho+M\right)\right)\sqrt{\frac T2\log\frac1\delta}$ term.\looseness-1
\end{Thm}
\begin{proof}[Proof sketch]
	At a high-level, we use the $\BigO(\log T)$ regret of follow-the-leader over Bregman divergences~\citep{khodak2019adaptive} for the initialization, the $\BigO(\sqrt{T\log k})$ regret of multiplicative weights over $k$ experts~\citep{shalev-shwartz2011oco} to tune over a large grid of hyperparameters, the fact that $U_t$ is an affine function of the Bregman divergence to combine the two methods, and the identity $\sum_{t=1}^T\Breg_\phi(\*x_t||\bar{\*x})=\sum_{t=1}^T\phi(\*x_t)-\phi(\bar{\*x})$ to bound the discretization error.
	The w.h.p. result follows by \citet[Lemma~4.1]{cesa-bianchi2006prediction}.
\end{proof}

Note that we keep details of the dependence on values like Lipschitz constants because they are important in applying this result;
however, in general setting $\rho=1/\sqrt[4]T$ in \eqref{eq:metareg} yields $\tilde\BigO(T^\frac34)$-regret.
While a slow rate, note that Algorithm~\ref{alg:meta} is learning a sequence of affine functions of Bregman divergences that are non-smooth and non-convex in-general.
Theorem~\ref{thm:meta} is an important structural result;
our main contributions to multi-armed and linear bandits follow by applying its instantiations for specific regularizers $\phi$ and hyperparameter sets $\Theta^\ast$.
We also believe Theorem~\ref{thm:meta} may be of independent interest as it holds for any choice of Bregman divergence beyond those we consider, and unlike past work \citep{khodak2019adaptive} allows for explicit control of non-smooth regularizers near the boundaries.
The theorem allows tuning the hyperparameters over user-specified intervals for $\beta$ and $\varepsilon$ and over an infinite interval for the step-size $\eta>0$.
Note that a similar result is straightforward to show for $\beta$ outside $[0,1]$ or for discrete rather than continuous set of hyperparameters.

%% file: mab.tex

\section{Multi-armed bandits}\label{sec:mab}

We now turn to our first application:
the multi-armed bandits problem.
In this setting at each round $i$ of task $t$ we take action $a_{t,i}\in[d]$ and observe loss $\ell_{t,i}(a_{t,i})\in[0,1]$.
As algorithms for MAB are probabilistic, we often sample methods from distributions $\*x\in\overline\K=\triangle_d$ in the $k$-simplex, thus making the inner product $\langle\ell_{t,i},\*x_{t,i}\rangle$ the expectation.

In this paper we use as a base-learner a generalization of the popular Exp3 method of \citet{auer2002exp3}, which runs multiplicative weights over unbiased estimators of the losses.
The first generalization is of the OMD regularizer, which for Exp3 is the negative Shannon entropy; 
we employ the negative Tsallis entropy $\phi_\beta(\*p)=\frac{1-\sum_{a=1}^d\*p^\beta(a)}{1-\beta}$ for $\beta\in[0,1]$, which was used by \citet{abernethy2015fighting} to improve the dependence of the regret on the dimension from $\BigO(\sqrt{dm\log d})$ to the optimal $\BigO(\sqrt{dm})$.
Note that $\phi_\beta$ recovers the Shannon entropy in the limit $\beta\to1$, and also that $\Breg_{\phi_\beta}(\*x||\cdot)$ is non-convex in the second argument, making ours the first known application of the online learnability of non-convex Bregman divergences.
The second generalization is in the loss estimators;
for $\gamma>0$ we employ $\hat\ell_{t,i}(a)=\frac{\ell_{t,i}(a)1_{a_{t,i}=a}}{\*x_{t,i}(a)+\gamma}$, where $\*x_{t,i}(a)$ is the probability of sampling $a$ on round $i$ of task $t$.
While this is an under-estimate of $\ell_{t,i}(a)$, its lower variance compared to the unbiased estimator---recovered by setting $\gamma=0$---allows \citet{neu2015explore} to obtain high probability bounds.

As the Tsallis entropy is non-smooth at the simplex boundary, learning Tsallis divergences will require the tools developed previously for initializing $\OMD$ in the interior of $\overline\K$.
We set $\K_\varepsilon=\{\*x\in\triangle_d:\min_a\*x(a)\ge\varepsilon/d\}$, so that the offset optimum $\hat{\*x}_t^{(\theta)}$ then has the very simple form $\OPT_\varepsilon(\hat\ell_t)=(1-\varepsilon)\hat{\*x}_t+\varepsilon\*1_d/d$, i.e. it is the mixture of the estimated optimum $\hat{\*x}_t$ over the entire simplex with the uniform distribution.
Note that for MAB we will {\em not} need to the capability of Algorithm~\ref{alg:meta} to learn $\varepsilon$ using multiplicative weights and can just set it assuming knowledge of the number of tasks.
Thus the method in this setting can be roughly summarized as doing the following at each task $t>1$:
\begin{enumerate}[noitemsep]
	\item sample $\theta_t=(\eta_t,\beta_t)$ from a distribution $\*p_t$ over the discretization $\Theta$
	\item run $\OMD_{\beta_t,\eta_t}$ using the initialization $\*x_{t,1}
	=\frac1{t-1}\sum_{s<t}\hat{\*x}_t^{(\theta_t)}
	=\frac\varepsilon d\*1_d+\frac{1-\varepsilon}{t-1}\sum_{s<t}\hat{\*x}_t$
	\item update $\*p_{t+1}$ using multiplicative weights with the expert losses $\frac1{\eta_t}\Breg_{\phi_{\beta_t}}(\hat{\*x}_t^{(\varepsilon)}||\*x_{t,1})+\frac{\eta_td^\beta_tm}{\beta_t}$
\end{enumerate}
The latter regret-upper-bound is derived from the within-task regret of OMD with the Tsallis regularizer.
This simple procedure achieves the following guarantee on the task-averaged regret:

\begin{Thm}\label{thm:mab}
	Suppose $\OMD_{\eta,\beta}$ is online mirror descent with the Tsallis entropy regularizer $\phi_\beta$ over $\gamma$-offset loss estimators.
	Then for every $\varepsilon>0$ and $\underline\beta\in(0,1]$ there exists integer $k=\tilde\BigO(\lceil d^4\sqrt{mT}\log\frac1\varepsilon\rceil)$ and $\alpha,\underline\eta,\overline\eta\in(0,\infty)$ such that running Algorithm~\ref{alg:meta} with $\Theta$ the product of uniform grids of size $k$ over each non-singleton dimension of $[\underline\eta,\overline\eta]\times[\max\{\underline\beta,1/\log d\},1]\times\{\varepsilon\}$ and $\alpha$ the meta-step-size yields w.p. at least $1-\delta$ the task-averaged regret
	\begin{align}
	\begin{split}
	\frac1T&\sum_{t=1}^T\sum_{i=1}^m\ell_{t,i}(a_{t,i})-\ell_{t,i}(a_t^\ast)\\
	&\le\tilde\BigO\left(
	\frac{\sqrt d}{\gamma T}\log\frac4\delta
	+\left(\frac\varepsilon{\gamma d}+\gamma d\right)m
	+\frac{d^{2-\underline\beta}\sqrt m}{\rho\varepsilon^{2-\underline\beta}T}
	+\left(\frac{\sqrt d}\rho+d\right)\sqrt{\frac{md}T\log\frac4\delta}
	+\rho d\sqrt m
	\right)\\
	&+\qquad\min_{\eta>0,\beta\in[\underline\beta,1]}\frac{H_\beta(\hat{\bar{\*x}})}\eta+\frac{\eta d^\beta m}\beta+\frac{\varepsilon^\beta d^{1-\beta}1_{\beta<1}}{(1-\beta)\eta}
	\end{split}	
	\end{align}
	where $H_\beta=-\phi_\beta$ is the Tsallis entropy and $\hat{\bar{\*x}}$ is the mean of the estimated optima $\hat{\*x}_1,\dots,\hat{\*x}_T$.
\end{Thm}

We see that the regret-upper-bound is highly dependent on the loss estimator offset $\gamma$, the boundary offset $\varepsilon$, the step-size offset $\rho$, and the lower bound $\underline\beta$ on the parameter of the Tsallis entropy.
Thus to clarify the guarantee we consider three regimes of $\underline\beta$:
$\underline\beta=1$, i.e. always using Exp3;
$\underline\beta=1/2$, which corresponds to the standard setting when using the Tsallis entropy \citep{abernethy2015fighting};
and $\underline\beta=1/\log d$, below which the OMD regret-upper-bound always worsens and so it does not make sense to try $\beta<1/\log d$.

\begin{Cor}\label{cor:mab}
	Suppose we run Algorithm~\ref{alg:meta} as in Theorem~\ref{thm:mab}.
	For $\underline\beta=1$, if we set $\varepsilon=\frac1{\sqrt T}$, $\gamma=\frac{\sqrt{\log\frac4\delta}}{d\sqrt[4]T}$, and $\rho=\frac{\sqrt d}{\sqrt[4]T}$ then w.p. $1-\delta$ the task-averaged regret satisfies
	\begin{equation}\label{eq:exp3}
	\tilde\BigO\left(\frac{d^\frac32+\sqrt m}{\sqrt[4]T}\sqrt{m\log\frac4\delta}\right)+2\sqrt{H_1(\hat{\bar{\*x}})dm}
	\end{equation}
	For $\underline\beta=\frac12$, if $\varepsilon=\sqrt{\frac dT}$, $\gamma=\frac{\sqrt{\log\frac4\delta}}{d\sqrt[4]T}$, and $\rho=\frac1{\sqrt{md}}$ then w.p. $1-\delta$ the task-averaged regret is
	\begin{equation}\label{eq:half}
	\tilde\BigO\left(\frac{dm\sqrt{d\log\frac4\delta}}{\sqrt[4]T}\right)+2\sqrt d+2\min_{\beta\in\left[\frac12,\frac{\log d-1}{\log d}\right]}\sqrt{H_\beta(\hat{\bar{\*x}})d^\beta m/\beta}
	\end{equation}
	For $\underline\beta=\frac1{\log d}$, if $\varepsilon=\frac1{\sqrt[3]T}$, $\gamma=\frac{\sqrt{\log\frac4\delta}}{d\sqrt[6]T}$, and $\rho=\frac{\sqrt d}{\sqrt[6]T}$ then w.p. $1-\delta$ the task-averaged regret is 
	\begin{equation}\label{eq:full}
	\tilde\BigO\left(\frac{d^\frac32+\sqrt m}{\sqrt[6]T}\sqrt{m\log\frac4\delta}\right)+2\min_{\beta\in(0,1]}\sqrt{H_\beta(\hat{\bar{\*x}})d^\beta m/\beta}+\sqrt{\frac{d1_{\beta<1}}{\beta(1-\beta)mT^\frac\beta 3}}
	\end{equation}
\end{Cor}

These results show that for all three settings of $\underline\beta$, as the meta-learner sees more tasks the average regret depends directly on the entropy of the estimated optima-in-hindsight, a natural notion of task-similarity since it is small if most tasks are estimated to be solved by the same arms and large if all arms are used roughly the same amount.
It also demonstrates how our algorithm's automatic tuning of the step-size $\eta$ allows us to set the asymptotic rate optimally depending on the entropy.
The algorithm's tuning of the entropy itself via $\beta$ also enables adaptation to similar tasks;
specifically, a smaller $\beta$ weights the $H_\beta(\hat{\bar{\*x}})/\eta$ term higher and is thus beneficial if tasks are similar.
As a natural example, suppose a constant $s\ll d$ actions are always minimizers, i.e. $\hat{\bar{\*x}}$ is $s$-sparse.
Then the last bound~\eqref{eq:full} implies that Algorithm~\ref{alg:meta} can achieve  task-averaged regret $o_T(1)+\BigO(\sqrt{sm\log d})$, albeit at the cost of slow convergence.
In-general, for the case of tuning over all $\beta\ge1/\log d$ the speed of the convergence depends on the optimal $\beta$;
the algorithm will converge very slowly at rate $\tilde\BigO(1/\sqrt[6\log d]T)$ if the optimal $\beta$ is around $1/\log d$, but for $\beta$ near 1 the rate will be $\tilde\BigO(1/\sqrt[4]T)$.
Note that we show in the intermediate case of tuning only as low as $\beta=1/2$ that we can still achieve $\tilde\BigO(1/\sqrt[4]T)$ at the cost of a fast $2\sqrt d$ term per-task.
Finally, note that because the entropy is bounded by $d^{1-\beta}$ we do asymptotically recover worst-case guarantees in all three cases if the tasks are dissimilar.

To put these results in some theoretical context, we can compare them to those of \citet{azizi2022non}, who achieve task-averaged regret bounds of the form $\tilde\BigO(1/\sqrt T+\sqrt{sm})$ in the {\em stochastic} MAB setting, where $s$ is an unknown subset of optimal actions.
Unlike their result, we study the harder adversarial setting and do {\em not} place restrictions on how the tasks are related;
despite this greater generality, our bounds are asymptotically comparable if the estimated and true optima-in-hindsight are roughly equivalent, as we also have $\tilde\BigO(\sqrt{sm})$ average regret as $T\to\infty$.
On the other hand, the rate in the number of tasks of \citet{azizi2022non} is much better, albeit at a cost of runtime exponential in $s$.
Apart from generality, we believe a great strength of our result is its adaptiveness;
unlike this work, we do not need to know how many optimal arms there are or their entropy in order to improve task-averaged regret with task-similarity.

%% file: blo.tex

\section{Bandit linear optimization}\label{sec:blo}

Our second general application is to bandit linear optimization,  in which at each round $i$ of task $t$ we play a vector $\*x_{t,i}\in\K$ for some convex set $\K$ and observe loss $\langle\ell_{t,i},\*x_{t,i}\rangle\in[-1,1]$.
We will again use a variant of mirror descent on top of estimated losses, this time setting $\phi$ to be a self-concordant barrier function with specialized loss estimators as in \citet{abernethy2008competing}.
This class of algorithms is picked because of its general applicability to any convex domain $\K$ via the construction of such barriers and the optimal dependence of its regret on the number of rounds $m$.
Note that our ability to handle non-smooth regularizers via the structural result in Theorem~\ref{thm:meta} is even more important here, as the barrier functions are  infinite at the boundaries.
Indeed, in this section we will no longer learn a $\beta$ parameterizing the regularizer and instead focus on learning an offset $\varepsilon>0$ away from the boundary.
For each such offset define $\K_\varepsilon=\{\*x\in\R^d:\pi_{\*x_{1,1}}(\*x)\le1/(1+\varepsilon)\}\subset\K$, where $\*x_{1,1}=\argmin_{\*x\in\K}\phi(\*x)$ and $\pi_{\*x_{1,1}}(\*x)=\inf_{\lambda\ge0,\*x_{1,1}+(\*x-\*x_{1,1})/\lambda\in\K}\lambda$ is the Minkowski function.
As before we obtain the $\varepsilon$-restricted optima-in-hindsight via the primitive $\OPT_\varepsilon(\hat\ell_t)=\argmin_{\*x\in\K_\varepsilon}\langle\hat\ell_t,\*x\rangle$.

With this specified, we can again adapt our meta-approach of Algorithm~\ref{alg:meta}, roughly summarized for BLO as doing the following at each task $t>1$:
\begin{samepage}
\begin{enumerate}[noitemsep]
	\item sample $\theta_t=(\eta_t,\varepsilon_t)$ from a distribution $\*p_t$ over the discretization $\Theta$
	\item run $\OMD_{\eta_t}$ using the initialization $\*x_{t,1}
	=\frac1{t-1}\sum_{s<t}\hat{\*x}_t^{(\theta_t)}
	=\frac1{t-1}\sum_{s<t}\OPT_{\varepsilon_t}(\hat\ell_t)$
	\item update $\*p_{t+1}$ using multiplicative weights with losses $\frac1{\eta_t}\Breg_\phi(\hat{\*x}_t^{(\varepsilon_t)}||\*x_{t,1})+(32d^2\eta_t+\varepsilon_t)m$
\end{enumerate}
\end{samepage}
Note that this algorithm is very similar to that for MAB, with both being special cases of Algorithm~\ref{alg:meta}, with the main difference being the different upper bound passed to multiplicative weights.
The procedure has the following guarantee

\begin{Thm}\label{thm:blo}
	Suppose $\OMD_{\eta,\beta}$ is online mirror descent with a self-concordant barrier $\phi$ as a regularizer and loss estimators specified as in \citet{abernethy2008competing}.
	Then for every $\overline\varepsilon\in(0,1/\sqrt m]$ and $\underline\varepsilon\in(0,\overline\varepsilon]$ there exists an integer $k=\BigO(D_{\underline\varepsilon}^2d\lceil\sqrt{mT}\rceil)$, where $D_{\underline\varepsilon}^2$ is a bound on $\Breg_\phi$ over $\K_{\underline\varepsilon}$, and $\alpha,\underline\eta,\overline\eta\in(0,\infty)$ such that running Algorithm~\ref{alg:meta} with $\Theta$ the product of uniform grids of size $k$ over each dimension of $[\underline\eta,\overline\eta]\times[\underline\epsilon,\overline\epsilon]$ and $\alpha$ the meta-step-size yields the expected task-averaged regret
	\begin{align}
	\begin{split}
	\E\frac1T\sum_{t=1}^T\sum_{i=1}^m\langle\ell_{t,i},\*x_{t,i}-\*x_t^\ast\rangle
	&\le72d\sqrt m\sqrt[4]T\left(D_{\underline\varepsilon}\sqrt{\frac mT\log k}+\frac{S_{\underline\varepsilon}K^2}{D_{\underline\varepsilon}T}(1+\log T)\right)\\
	&\qquad+\min_{\*x\in\K,\eta>0,\varepsilon\in[\underline\varepsilon,\overline\varepsilon]}\E\frac1T\sum_{t=1}^T\frac{\Breg_\phi(\OPT_\varepsilon(\hat\ell_t)||\*x)}\eta+(32\eta d^2+\varepsilon)m\\
	&=\tilde\BigO\left(\frac{D_{\underline\varepsilon}dm}{\sqrt[4]T}+\frac{S_{\underline\varepsilon}K^2d\sqrt m}{D_{\underline\varepsilon}T^\frac34}\right)+\min_{\*x\in\K,\varepsilon\in[\underline\varepsilon,\overline\varepsilon]}4d\hat V_\varepsilon\sqrt{2m}+\varepsilon m
	\end{split}
	\end{align}
	where $S_{\underline\varepsilon}=\max_{\*x\in\K_{\underline\varepsilon}}\|\nabla^2\phi(\*x)\|_2$, $K$ is the Euclidean diameter of $\K$, and $\hat V_\varepsilon$ is what we call the {\bf barrier-divergence at level $\varepsilon$} defined by $\hat V_\varepsilon^2=\min_{\*x\in\K}\E\frac1T\sum_{t=1}^T\Breg_\phi(\OPT_\varepsilon(\hat\ell_t)||\*x)$.
\end{Thm}

For self-concordant barriers we generally have $D_{\underline\varepsilon}=\BigO(1/\underline\varepsilon)$ and $S_{\underline\varepsilon}=\BigO(1/\underline\varepsilon^2)$~\citep{abernethy2008competing}, so setting $\underline\varepsilon=1/m$ and yields
\begin{equation}
\E\frac1T\sum_{t=1}^T\sum_{i=1}^m\langle\ell_{t,i},\*x_{t,i}-\*x_t^\ast\rangle
\le\tilde\BigO\left(\frac{dm^2}{\sqrt[4]T}\right)+\min_{\frac1m\le\varepsilon\le\frac1{\sqrt m}}4d\hat V_\varepsilon\sqrt{2m}+\varepsilon m
\end{equation}
As before, this shows that as the number of tasks $T\to\infty$ the average regret improves with a notion of task-similarity $\hat V_\varepsilon$ that decreases if the estimated task-optima are close together.
Roughly speaking, if tasks have barrier-divergence $\hat V_\varepsilon$ then the average regret will be $\BigO(\hat V_\varepsilon\sqrt m+\varepsilon m)$, which can be a significant improvement over the single-task case, e.g. if $\hat V_\frac1m$ is small.
In-particular, our analysis removes explicit dependence on the square root of the self-concordance constant of $\phi$ in the single-task case \citep{abernethy2008competing};
as an example, this constant is equal to the number of constraints if $\K$ is defined by linear inequalities, as in the bandit shortest-path application below.
Note that the use of $\varepsilon$-constrained optima is necessary for this problem due to the regularizers being infinite at the boundaries, where all true optima lie.\looseness-1

To make the above result and task-similarity notion more concrete, consider the following corollary for BLO over the unit sphere $\K=\{\*x\in\R^d:\|\*x\|_2\le1\}$:
\begin{Cor}\label{cor:sphere}
	Let $\K$ be the unit sphere with the self-concordant barrier $\phi(\*x)=-\log(1-\|\*x\|_2^2)$.
	Then Algorithm~\ref{alg:meta} attains expected task-averaged regret bounded by
	\begin{equation}
		\tilde\BigO\left(\frac{dm^2}{\sqrt[4]T}\right)+\min_{\frac1m\le\varepsilon\le\frac1{\sqrt m}}4d\E\sqrt{2m\log\left(\frac{1-\|\hat{\bar\ell}^{(\varepsilon)}\|_2^2}{2\varepsilon-\varepsilon^2}\right)}+\varepsilon m
	\end{equation}
	for $\hat{\bar\ell}^{(\varepsilon)}=\frac1T\sum_{t=1}^T\OPT_\varepsilon(\hat\ell_t)=\frac{\varepsilon-1}T\sum_{t=1}^T\frac{\hat\ell_t}{\|\hat\ell_t\|_2}$ the average over normalized estimated task-optima.\looseness-1
\end{Cor}
Thus in this setting if all tasks have similar estimated losses then $\hat{\bar\ell}^{(\varepsilon)}$ will be an average over similar vectors and thus have large Euclidean norm close to $1-\varepsilon$, making the term in the logarithm above close to 1.
In this case $\hat V_\varepsilon$ is close to zero and so the average regret is $\varepsilon m$ as $T\to\infty$;
setting $\varepsilon=1/m$ yields constant asymptotic averaged regret.
This demonstrates the usefulness of the barrier-divergence as a measure of task-similarity.

As a final application, we apply our meta-BLO result to the shortest-path problem in online optimization \citep{takimoto2003path,kalai2005efficient}.
In its bandit variant \citep{awerbuch2004adaptive,dani2008price}, at each time step $i=1,\dots,m$ the player must choose a path $p_i$ from a fixed source $u\in V$ to a fixed sink $v\in V$ in a directed graph $G(V,E)$.
At the same time the adversary chooses edge weights $\ell_i\in\R^{|E|}$ and the player suffers the sum $\sum_{e\in p_t}\ell_i(e)$ of the weights in their chosen path $p_t$.
This can be transformed into BLO over vectors $\*x$ in a convex set $\K\subset[0,1]^{|E|}$ defined by a set $\C$ of $\BigO(|E|)$ linear constraints $(\*a,b)$ s.t. $\langle\*a,\*x\rangle\le b$ enforcing flows from $u$ to $v$;
paths from $u$ to $v$ can then be sampled from any $\*x\in\K$ in an unbiased manner \citep[Proposition~1]{abernethy2008competing}.
In the single-task case the BLO method of \citet{abernethy2008competing} yields an $\BigO(|E|^\frac32\sqrt m)$-regret algorithm for this problem.

In the multi-task case consider a sequence of $t=1,\dots,T$ shortest path instances, each consisting of $m$ edge loss vectors $\ell_{t,i}$ selected by an adversary.
The goal is to minimize average regret across instances.
Note that our setup may be viewed as learning a prediction of the optimal path in a manner similar to the algorithms with predictions paradigm in beyond-worst-case-analysis \citep{mitzenmacher2021awp};
in-particular, we have incorporated predictions into the algorithm of \citet{abernethy2008competing} via the meta-initialization approach and now present the learning-theoretic result for an end-to-end guarantee \citep{khodak2022awp}.
\begin{Cor}\label{cor:path}
	Let $\K=\{\*x\in[0,1]^{|E|}:\langle\*a,\*x\rangle\le b~\forall~(\*a,b)\in\C\}$ be the set of flows from $u$ to $v$ on a graph $G(V,E)$, where $\C\subset\R^{|E|}\times\R$ is a set of $\BigO(|E|)$ linear constraints.
	Suppose we see $T$ instances of the bandit online shortest path problem with $m$ timesteps each.
	Then sampling from probability distributions over paths from $u$ to $v$ returned by running Algorithm~\ref{alg:meta} with regularizer $\phi(\*x)=-\sum_{\*a,b\in\C}\log(b-\langle\*a,\*x\rangle)$ attains the following expected average regret across instances:
	\begin{equation}
		\tilde\BigO\left(\frac{|E|m^2}{\sqrt[4]T}\right)+\min_{\frac1m\le\varepsilon\le\frac1{\sqrt m}}4|E|\E\sqrt{2m\sum_{\*a,b\in\C}\log\left(\frac{\frac1T\sum_{t=1}^Tb-\langle\*a,\hat{\*x}_t^{(\varepsilon)}\rangle}{\sqrt[T]{\prod_{t=1}^Tb-\langle\*a,\hat{\*x}_t^{(\varepsilon)}\rangle}}\right)}+\varepsilon m
	\end{equation}
	Here $\hat{\*x}_t^{(\varepsilon)}=\OPT_\varepsilon(\hat\ell_t)$ is the $\varepsilon$-constrained estimated optimal flow for instance $t$.
\end{Cor}
Corollary~\ref{cor:path} shows that the average regret on the $T$ bandit shortest-path problems scales with the sum across all constraints $\*a,b\in\C$ of the log of the ratio between the arithmetic and geometric mean of the distances $b-\langle\*a,\hat{\*x}_t^{(\varepsilon)}\rangle$ from the estimated optimum flow $\hat{\*x}_t^{(\varepsilon)}$ to the constraint boundary.
Since the arithmetic and geometric mean are equal exactly when all entries are equal---and otherwise the former is larger---this means that the regret is small when the estimated optimal flows $\hat{\*x}_t^{(\varepsilon)}$ for each task are at similar distances from the constraints.

%% file: conclusion.tex

\section{Conclusion}

In this work, we develop a meta-algorithm for learning to initialize and tune OMD for regularizers used in adversarial bandit tasks. We apply our meta-algorithm to obtain task-averaged regret guarantees for both the multi-armed and the linear bandit settings that depend on natural, setting-specific notions of task similarity. For MAB, we use OMD with the Tsallis regularizer as our base-learner and meta-learn the initialization, step-size, and entropy parameter. For BLO, we again use a variant of mirror descent with self-concordant barrier regularizers as our base-learner, and meta-learn the initialization, step-size, and boundary-offset. 

\newpage
There are several exciting directions for future work.
A limitation of our current results is the dependence of the task-similarity on optima estimated by the within-task algorithm;
while they have the benefit of being computable, one may wish to obtain task-averaged regret bounds which do not depend on algorithmic quantities.
In particular, defining task-similarity via the true optima or losses may be more natural.
Achieving this may be possible by making further assumptions on the structure of the problem, e.g. gap conditions or best-arm identifiability.
Another direction is to extend these results to other adversarial bandits settings such as contextual bandits and Lipschitz bandits.

%% file: proofs.tex

\newpage
\section{Proof of Theorem~\ref{thm:meta}}

\begin{proof}
	We define $\underline\eta=\frac{\rho D}{G\sqrt m}$, $\overline\eta=\frac{2DM}{G\sqrt m}$, number of grid points $k=\Omega(\lceil(4D^2M^2LG+C)\sqrt{mT}\rceil)$,  $\Theta=\left\{\underline\eta+\frac jk(\overline\eta-\underline\eta)\right\}_{j=0}^k\times\{\underline\beta+\frac jk(\overline\beta-\underline\beta)\}_{j=0}^{k1_{\overline\beta>\underline\beta}}\times\{\underline\varepsilon+\frac jk(\overline\varepsilon-\underline\varepsilon)\}_{j=0}^{k1_{\overline\varepsilon>\underline\varepsilon}}$, and meta-step-size $\alpha=\frac1{DG/\rho+2DMG+C\sqrt m}\sqrt{\frac{3\log k}{2Tm}}$.
	Note that
	\begin{equation}
		\frac{\rho D}{G\sqrt m}
		\le\argmin_{\eta>0}\min_{\*x\in\K_{\underline\varepsilon},\beta,\varepsilon}\sum_{t=1}^T\tilde U_t(\*x,(\eta,\beta,\varepsilon))
		\le\frac{DM}G\sqrt{\frac{1+\rho^2}m}
		\le\frac{2DM}{G\sqrt m}
	\end{equation}
	so 
	\begin{equation}
		\max_{t\in[T]}U_t^{(\rho)}(\*x_t^{(\theta_t)},\theta_t)
		\le\frac{DG\sqrt m}\rho+2DMG\sqrt{m}+Cm
	\end{equation}
	Therefore applying the regret guarantee for exponentiated gradient \citep[Corollary~2.14]{shalev-shwartz2011oco} followed by the regret of follow-the-leader on a sequent of Bregman divergences (Lemma~\ref{lem:bregman}) yields
	\begin{align}
	\begin{split}
		\E&\sum_{t=1}^TU_t(\*x_t^{(\theta_t)},\theta_t)\\
		&\le\E\sum_{t=1}^TU_t^{(\rho)}(\*x_t^{(\theta_t)},\theta_t)\\
		&\le\left(C\sqrt m+DG\left(\frac1\rho+2M\right)\right)\sqrt{2mT\log|\Theta|}+\min_{\theta\in\Theta}\E\sum_{t=1}^TU_t^{(\rho)}(\*x_t^{(\theta)},\theta)\\
		&\le\left(C\sqrt m+DG\left(\frac1\rho+2M\right)\right)\sqrt{2mT\log|\Theta|}\\
		&\qquad+\min_{(\eta,\beta,\varepsilon)\in\Theta}\frac{8SK^2}\eta(1+\log T)+\min_{\*x\in\K_\varepsilon}\E\sum_{t=1}^T\frac{\Breg_{\phi_\beta}(\hat{\*x}_t^{(\varepsilon)}||\*x)+\rho^2D^2}\eta+(\eta G_\beta^2+C\varepsilon)m\\
		&\le\left(C\sqrt m+DG\left(\frac1\rho+2M\right)\right)\sqrt{2mT\log|\Theta|}+\frac{8SK^2\overline G\sqrt m}{\rho D}(1+\log T)\\
		&\qquad+\min_{(\eta,\beta,\varepsilon)\in\Theta}\eta G_\beta^2mT+C\varepsilon mT+\frac{\rho^2D^2T}\eta+\E\sum_{t=1}^T\frac{\phi_\beta(\hat{\*x}_t^{(\varepsilon)})-\phi_\beta(\hat{\bar{\*x}}^{(\varepsilon)})}\eta\\
		&\le\left(C\sqrt m+DG\left(\frac1\rho+2M\right)\right)\sqrt{2mT\log|\Theta|}+\frac{8SK^2G\sqrt m}{\rho D}(1+\log T)+\rho DGT\sqrt m\\
		&\qquad+\left(4D^2M^2+\left(C\sqrt m+\frac{2LG}{\rho D}\right)(\overline\varepsilon-\underline\varepsilon)\sqrt m+2\left(\frac{2M}G+\frac G\rho\right)DL\sqrt m(\overline\beta-\underline\beta)\right)\frac Tk\\
		&\qquad+\min_{(\eta,\beta,\varepsilon)\in\overline\Theta}\eta G_\beta^2mT+C\varepsilon mT
		+\E\sum_{t=1}^T\frac{\phi_\beta(\hat{\*x}_t^{(\varepsilon)})-\phi_\beta(\hat{\bar{\*x}}^{(\varepsilon)})}\eta\\
		&\le\left(C\sqrt m+DG\left(\frac1\rho+2M\right)\right)\sqrt{6mT\log k}+(4D^2M^2LG+C)\frac{Tm}{\rho k}\\
		&\qquad+\frac{8SK^2G\sqrt m}{\rho D}(1+\log T)+\rho DGT\sqrt m+\min_{\*x\in\K,\theta\in\Theta^\ast}\E\sum_{t=1}^TU_t(\*x,\theta)
	\end{split}
	\end{align}
	where the fourth inequality follows by Claim~\ref{clm:bregman}, the fifth by Lipschitzness of $1/\eta$ on $\eta\ge\frac{\rho D}{G\sqrt m}$ and of $\phi_\beta(\hat{\*x}_t^{(\varepsilon)}||\cdot)$ on $\K_{\underline\varepsilon}$, and the sixth by simplifying and substituting the lower bound for $k$.
	The w.h.p. version of the bound follows by applying \citet[Lemma~4.1]{cesa-bianchi2006prediction} when obtaining the second inequality.
\end{proof}

\begin{Lem}\label{lem:bregman}
	Let $\phi:\K\mapsto\R_{\ge0}$ be a strictly-convex function with $\max_{\*x\in\K}\|\nabla^2\phi(\*x)\|_2\le S$ over a convex set $\K\subset\R^d$ with $\max_{\*x\in\K}\|\*x\|_2\le K$.
	Then for any points $\*x_1,\dots,\*x_T\in\K$ the actions $\*y_1=\argmin_{\*x\in\K}\phi(\*x)$ and $\*y_t=\frac1{t-1}\sum_{s<t}\*x_s$ have regret
	\begin{equation}
		\sum_{t=1}^T\Breg_\phi(\*x_t||\*y_t)-\Breg_\phi(\*x_t||\*y_{T+1})
		\le\sum_{t=1}^T\frac{8SK^2}{2t-1}
		\le8SK^2(1+\log T)
	\end{equation}
\end{Lem}
\begin{proof}
	Note that
	\begin{equation}
		\nabla_{\*y}\Breg_\phi(\*x||\*y)
		=-\nabla\phi(\*y)-\nabla_{\*y}\langle\nabla\phi(\*y),\*x\rangle+\nabla_{\*y}\langle\nabla\phi(\*y),\*y\rangle
		=\diag(\nabla^2\phi(\*y))(\*y-\*x)
	\end{equation}
	so $\Breg_\phi(\*x_t||\*y)$ is $2SK$-Lipschitz w.r.t. the Euclidean norm.
	Applying \citet[Proposition~B.1]{khodak2019adaptive}  yields the result.
\end{proof}

\begin{Clm}\label{clm:bregman}
		Let $\phi:\K\mapsto\R$ be a strictly-convex function over a convex set $\K\subset\R^d$ containing points $\*x_1,\dots,\*x_T$.
		Then their mean $\bar{\*x}=\frac1T\sum_{t=1}^T\*x_t$ satisfies
		\begin{equation}
			\sum_{t=1}^T\Breg_\phi(\*x_t||\bar{\*x})
			=\sum_{t=1}^T\phi(\*x_t)-\phi(\bar{\*x})
		\end{equation}
\end{Clm}
\begin{proof}
	\begin{align}
	\begin{split}
	\sum_{t=1}^T\Breg_\phi(\*x_t||\bar{\*x})
	&=\sum_{t=1}^T\phi(\*x_t)-\phi(\bar{\*x})-\langle\nabla\phi(\bar{\*x}),\*x_t-\bar{\*x}\rangle\\
	&=\sum_{t=1}^T\phi(\*x_t)-\phi(\bar{\*x})-\langle\nabla\phi(\bar{\*x}),\sum_{t=1}^T\*x_t-\bar{\*x}\rangle\\
	&=\sum_{t=1}^T\phi(\*x_t)-\phi(\bar{\*x})
	\end{split}
	\end{align}
\end{proof}

\newpage
\section{Proof of Theorem~\ref{thm:mab}}
\begin{proof}
	Since $\varepsilon$ is constant we use the shorthand $\hat{\*x}_t^{(\varepsilon)}=\hat{\*x}_t^{(\theta)}~\forall~\theta\in\Theta$.
	Note that we use search space $\Theta=\left[\frac\rho{\sqrt m},2\sqrt{\frac{d\log d}{em}}\right]\times\left[\max\left\{\frac1{\log d},\underline\beta\right\},1\right]\times\{\varepsilon\}$.
	We have the constants $D\le\sqrt d$, $G\le\sqrt d$, $M\le\sqrt{\frac{d\log d}e}$, $S\le\left(\frac d \varepsilon\right)^{2-\underline\beta}$, and $K=1$.
	Note that the second term $d^{1-\beta}m/\beta$ is decreasing on $\beta<1/\log d$, so since $\phi_\beta$ is always increasing in $\beta$ we know that the optimal $\beta$ is in $[1/\log d,1]$.
	Note that by Lemma~\ref{lem:tsallis} we have that $L=d\log\frac d\varepsilon$.
	We thus have
	\begin{align}
	\begin{split}
	\sum_{t=1}^T&\sum_{i=1}^m\ell_{t,i}(a_{t,i})-\ell_{t,i}(a_t^\ast)\\
	&\le\sum_{t=1}^T\sum_{i=1}^m\langle\hat\ell_{t,i},\*x_{t,i}-\ell_{t,i}(a_t^\ast)\rangle+\gamma\sum_{a=1}^d\hat\ell_{t,i}(a)\\
	&\le\sum_{t=1}^T\frac{\Breg_{\phi_{\beta_t}}(\hat{\*x}_t^{(\varepsilon)}||\*x_{t,1})}{\eta_t}+\sum_{i=1}^m\langle\hat\ell_{t,i},\hat{\*x}_{t,i}^{(\varepsilon)}\rangle-\ell_{t,i}(a_t^\ast)+\frac{\eta_t}{\beta_t}\sum_{a=1}^d\*x_{t,i}^{2-\beta_t}(a)\hat\ell_{t,i}^2(a)+\gamma\sum_{a=1}^d\hat\ell_{t,i}(a)\\
	&\le\frac{\varepsilon mT}{\gamma d}+\sum_{t=1}^T\frac{\Breg_{\phi_{\beta_t}}(\hat{\*x}_t^{(\varepsilon)}||\*x_{t,1})}{\eta_t}+\sum_{i=1}^m\hat\ell_{t,i}(a_t^\ast)-\ell_{t,i}(a_t^\ast)\\
	&\qquad+\sum_{t=1}^T\frac{\eta_t}{\beta_t}\sum_{i=1}^m\sum_{a=1}^d\*x_{t,i}^{1-\beta_t}(a)\hat\ell_{t,i}(a)+\gamma\sum_{a=1}^d\hat\ell_{t,i}(a)\\
	&\le\frac{\varepsilon mT}{\gamma d}+\frac{1+\frac{\overline\eta}{\underline\beta}+\gamma}{2\gamma}\log\frac4\delta+\sum_{t=1}^T\frac{\Breg_{\phi_{\beta_t}}(\hat{\*x}_t^{(\varepsilon)}||\*x_{t,1})}{\eta_t}\\
	&\qquad+\sum_{t=1}^T\frac{\eta_t}{\beta_t}\sum_{i=1}^m\sum_{a=1}^d\*x_{t,i}^{1-\beta_t}(a)\ell_{t,i}(a)+\gamma\sum_{a=1}^d\ell_{t,i}(a)\\
	&\le\frac{\varepsilon mT}{\gamma d}+\frac{1+\sqrt{\frac{d\log^3d}{em}}}\gamma\log\frac4\delta+\gamma dmT+\sum_{t=1}^T\frac{\Breg_{\phi_{\beta_t}}(\hat{\*x}_t^{(\varepsilon)}||\*x_{t,1})}{\eta_t}+\frac{\eta_td^{\beta_t}m}{\beta_t}\\
	&\le\frac{\varepsilon mT}{\gamma d}+\frac{1+\sqrt{\frac{d\log^3d}{em}}}\gamma\log\frac4\delta+\gamma dmT+\min_{\*x\in\triangle_d,\eta>0,\beta\in[\underline\beta,1]}\sum_{t=1}^T\frac{\Breg_{\phi_\beta}(\hat{\*x}_t^{(\varepsilon)}||\*x)}\eta+\frac{\eta d^\beta m}\beta\\
	&\qquad+2d\left(\frac1\rho+\sqrt{\frac{d\log d}e}\right)\sqrt{6mT\log\frac{4k}\delta}+\frac{8d^{2-\underline\beta}\sqrt{m}}{\rho\varepsilon^{2-\underline\beta}}(1+\log T)+\rho dT\sqrt m\\
	&\le\left(\frac\varepsilon{\gamma d}+\gamma d\right)mT+\frac{1+\sqrt{\frac{d\log^3d}{em}}}\gamma\log\frac4\delta+T\min_{\eta>0,\beta\in[\underline\beta,1]}\frac{H_\beta(\hat{\bar{\*x}})}\eta+\frac{\eta d^\beta m}\beta+\frac{\varepsilon^\beta d^{1-\beta}1_{\beta<1}}{(1-\beta)\eta}\\
	&\qquad+2d\left(\frac1\rho+\sqrt{\frac{d\log d}e}\right)\sqrt{6mT\log\frac{4k}\delta}+\frac{8d^{2-\underline\beta}\sqrt{m}}{\rho\varepsilon^{2-\underline\beta}}(1+\log T)+\rho dT\sqrt m
	\end{split}
	\end{align}
	where the second inequality follows by Lemma~\ref{lem:mirror}, the third by H\"older's inequality and the definitions $\hat\ell_{t,i}$ and $\hat{\*x}_{t,i}^{(\varepsilon)}$, the fourth by \citet[Lemma~1]{neu2015explore}, the fifth by the definition of $\ell_{t,i}$, the sixth by Theorem~\ref{thm:meta}, and the last by the derivation below for $\beta<1$ (otherwise it holds by joint convexity of the KL-divergence) followed by Claim~\ref{clm:bregman} combined with the fact that the entropy of optima-in-hindsight is zero.
	\begin{align}
	\begin{split}
	-\phi_\beta((1-\varepsilon)\*x+\varepsilon\*1_d/d)
	&=\frac{\sum_{a=1}^d((1-\varepsilon)\*x(a)+\varepsilon/d)^\beta-1}{1-\beta}\\
	&\le\frac{\varepsilon^\beta d^{1-\beta}+(1-\varepsilon)^\beta\sum_{a=1}^d\*x^\beta(a)-1}{1-\beta}
	\le\frac{\varepsilon^\beta d^{1-\beta}}{1-\beta}
	\end{split}
	\end{align}
\end{proof}

\begin{Lem}\label{lem:mirror}
	Suppose we play $\OMD_{\beta,\eta}$ with regularizer $\phi_\beta$ the negative Tsallis entropy and initialization $\*x_1\in\triangle_d$ on the sequence of linear loss functions  $\ell_1,\dots,\ell_T\in[0,1]^d$.
	Then for any $\*x^\ast\in\triangle_d$ we have
	\begin{equation}
		\sum_{t=1}^T\langle\ell_t,\*x_t-\*x^\ast\rangle
		\le\frac{\Breg_{\phi_\beta}(\*x^\ast||\*x_1)}\eta+\frac\eta\beta\sum_{a=1}^d\*x_t^{2-\beta}(a)\ell_t^2(a)
	\end{equation}
\end{Lem}
\begin{proof}
	Note that the following proof follows parts of the course notes by \citet{luo2017tsallis}, which we reproduce for completeness.
	The OMD update at each step $t$ involves the following two steps: set $\*y_{t+1}\in\triangle_d$ s.t. $\nabla\phi_\beta(\*y_{t+1})=\nabla\phi_\beta(\*x_t)-\eta\ell_t$ and then set $\*x_{t+1}=\argmin_{\*x\in\triangle_d}\Breg_{\phi_\beta}(\*x,\*y_{t+1})$ \citep[Algorithm~14]{hazan2015oco}.
	Note that by \citet[Equation~5.3]{hazan2015oco} and nonnegativity of the Bregman divergence we have
	\begin{equation}
		\sum_{t=1}^T\langle\ell_t,\*x_t-\*x^\ast\rangle
		\le\frac{\Breg_{\phi_\beta}(\*x^\ast||\*x_1)}\eta+\frac1\eta\sum_{t=1}^T\Breg_{\phi_\beta}(\*x_t||\*y_{t+1})
	\end{equation}
	To bound the second term, note that when $\phi_\beta$ is the negative Tsallis entropy we have
	\begin{align}
	\begin{split}
	\Breg_{\phi_\beta}&(\*x_t||\*y_{t+1})\\
	&=\frac1{1-\beta}\sum_{a=1}^d\left(\*y_{t+1}^\beta(a)-\*x_t^\beta(a)+\frac\beta{\*y_{t+1}^{1-\beta}(a)}(\*x_t(a)-\*y_{t+1}(a)\right)\\
	&=\frac1{1-\beta}\sum_{a=1}^d\left((1-\beta)\*y_{t+1}^\beta(a)-\*x_t^\beta(a)+\beta\left(\frac1{\*x_t^{1-\beta}(a)}+\frac{1-\beta}\beta\eta\ell_t(a)\right)\*x_t(a)\right)\\
	&=\sum_{a=1}^d\left(\*y_{t+1}^\beta(a)-\*x_t^\beta(a)+\eta\*x_t(a)\ell_t(a)\right)
	\end{split}
	\end{align}
	Plugging the following result, which follows from $(1+x)^\alpha\le1+\alpha x+\alpha(\alpha-1)x^2~\forall~x\ge0,\alpha<0$, into the above yields the desired bound.
	\begin{align}
	\begin{split}
	\*y_{t+1}^\beta(a)
	=\*x_t^\beta(a)\left(\frac{\*y_{t+1}^{\beta-1}(a)}{\*x_t^{\beta-1}(a)}\right)^\frac\beta{\beta-1}
	&=\*x_t^\beta(a)\left(1+\frac{1-\beta}\beta\eta\*x_t^{1-\beta}(a)\ell_t(a)\right)^\frac\beta{\beta-1}\\
	&\le\*x_t^\beta(a)\left(1-\eta\*x_t^{1-\beta}(a)\ell_t(a)+\frac{\eta^2}\beta\*x_t^{2-2\beta}(a)\ell_t(a)^2\right)\\
	&=\*x_t^\beta(a)-\eta\*x_t(a)\ell_t(a)+\frac{\eta^2}\beta\*x_t^{2-\beta}(a)\ell_t(a)^2
	\end{split}
	\end{align}
\end{proof}

\newpage
\begin{Lem}\label{lem:tsallis}
	For any $\rho\in(0,1/d]$ and $\*x\in\triangle_d$ s.t. $\*x(a)\ge\rho~\forall~a\in[d]$ the $\beta$-Tsallis entropy $H_\beta(\*x)=-\frac{1-\sum_{a=1}^d\*x^\beta(a)}{1-\beta}$ is $d\log\frac1\rho$-Lipschitz w.r.t. $\beta\in[0,1]$.
\end{Lem}
\begin{proof}
	Let $\log_\beta x=\frac{x^{1-\beta}-1}{1-\beta}$ be the $\beta$-logarithm function and note that by \citet[Equation~6]{yamano2002tsallis} we have
	$\log_\beta x-\log x=(1-\beta)(\partial_b\log_\beta x+\log_\beta x\log x)\ge0~\forall~\beta\in[0,1]$.
	Then we have for $\beta\in[0,1)$ that
	\begin{align}
	\begin{split}
	|\partial_\beta H_\beta(\*x)|
	&=\left|\frac{-H_\beta(\*x)-\sum_{a=1}^d\*x^\beta(a)\log\*x(a)}{1-\beta}\right|\\
	&=\frac1{1-\beta}\left|\sum_{a=1}^d\*x^\beta(a)(\log_\beta\*x(a)-\log\*x(a))\right|\\
	&=\frac1{1-\beta}\sum_{a=1}^d\*x^\beta(a)(\log_\beta\*x(a)-\log\*x(a))\\
	&\le\frac1{1-\beta}\left(\sum_{a=1}^d\*x(a)\right)^\beta\left(\sum_{a=1}^d(\log_\beta\*x(a)-\log\*x(a))^{\frac1{1-\beta}}\right)^{1-\beta}\\
	&\le\frac1{1-\beta}\sum_{a=1}^d\log_\beta\*x(a)-\log\*x(a)\\
	&\le\frac d{1-\beta}(\log_\beta\rho-\log\rho)\\
	&\le-d\log\rho
	\end{split}
	\end{align}
	where the fourth line follows by H\"older's inequality, the fifth by subadditivity of $x^a$ for $a\in(0,1]$, the sixth by the fact that $\partial_x(\log_\beta x-\log x)=x^{-\beta}-1/x\le0~\forall~\beta,x\in[0,1)$, and the last line by substituting $\beta=0$ since $\partial_\beta\left(\frac{\log_\beta\rho-\log\rho}{1-\beta}\right)
	=\frac{2(\rho-\rho^\beta)-(1-\beta)(\rho^\beta+\rho)\log\rho}{\rho^\beta(1-\beta)^3}\le0~\forall~\beta\in[0,1),\rho\in(0,1/d]$.
	For $\beta=1$, applying L'H\^opital's rule yields
	\begin{equation}
	\lim_{\beta\to1}\partial_\beta H_\beta(\*x)
	=-\frac12\lim_{\beta\to1}\sum_{a=1}^d\*x^\beta(a)\log^2\*x(a)(1-(1-\beta)\log\*x(a))
	=-\frac12\sum_{a=1}^d\*x(a)\log^2\*x(a)
	\end{equation}
	which is bounded on $[-2d/e^2,0]$.
\end{proof}

\newpage
\section{Proof of Theorem~\ref{thm:blo}}
\begin{proof}
	Applying Theorem~\ref{thm:meta} with constants $D=D_{\underline\varepsilon}$, $G=4d\sqrt 2$, $M=1$, $S=S_{\underline\varepsilon}$, and $K=K$ yields
	\begin{align}
	\begin{split}
	\E&\sum_{t=1}^T\sum_{i=1}^m\langle\ell_{t,i},\*x_{t,i}-\*x_t^\ast\rangle\\
	&\le\E\sum_{t=1}^T\varepsilon_tm+\sum_{i=1}^m\langle\ell_{t,i},\*x_{t,i}-\OPT_{\varepsilon_t}(\ell_t)\rangle\\
	&=\E\sum_{t=1}^T\varepsilon_tm+\sum_{i=1}^m\langle\hat\ell_{t,i},\*x_{t,i}-\OPT_{\varepsilon_t}(\ell_t)\rangle\\
	&\le\E\sum_{t=1}^T\varepsilon_tm+\sum_{i=1}^m\langle\hat\ell_{t,i},\*x_{t,i}-\hat{\*x}_t^{(\theta_t)}\rangle\\
	&\le\E\sum_{t=1}^T\frac{\Breg_\phi(\hat{\*x}_t^{(\theta_t)}||\*x_{t,1}^{(\theta_t)})}{\eta_t}+(\eta_tG^2+\varepsilon_t)m\\
	&\le\left(\sqrt m+\frac{4D_{\underline\varepsilon}G}\rho\right)\sqrt{6mT\log k}+\frac{8S_{\underline\varepsilon}K^2G\sqrt m}{\rho D_{\underline\varepsilon}}(1+\log T)+\rho D_{\underline\varepsilon}GT\sqrt m\\
	&\qquad+\min_{\*x\in\K,\eta>0,\varepsilon\in[\underline\varepsilon,\overline\varepsilon]}\E\sum_{t=1}^T\frac{\Breg_\phi(\OPT_\varepsilon(\hat\ell_t)||\*x)}\eta+(\eta G^2+\varepsilon)m\\
	&\le72d\sqrt m\sqrt[4]T\left(D_{\underline\varepsilon}\sqrt{mT\log k}+\frac{S_{\underline\varepsilon}K^2}{D_{\underline\varepsilon}}(1+\log T)\right)\\
	&\qquad+\min_{\*x\in\K,\eta>0,\varepsilon\in[\underline\varepsilon,\overline\varepsilon]}\E\sum_{t=1}^T\frac{\Breg_\phi(\OPT_\varepsilon(\hat\ell_t)||\*x)}\eta+(32\eta d^2+\varepsilon)m\\
	\end{split}
	\end{align}
	where the third inequality follows from Lemma~\ref{lem:concordant}.
\end{proof}

\begin{Lem}\label{lem:concordant}
	Let $\overline\K\subset\R^d$ be a convex set and $\phi$ be a self-concordant barrier.
	Suppose $\ell_1,\dots,\ell_T$ are a sequence of loss functions satisfying $|\langle\ell_t,\*x\rangle|\le1~\forall~\*x\in\K$.
	Then if we run OMD with step-size $\eta>0$ as in \citet[Algorithm~1]{abernethy2008competing} on the sequence of estimators $\hat\ell_t$ our estimated regret w.r.t. any $\*x^\ast\in\K_\varepsilon$ for $\varepsilon>0$ will satisfy
	\begin{equation}
	\sum_{t=1}^T\langle\hat\ell_t,\*x_t-\*x^\ast\rangle\le\frac{\Breg_\phi(\*x^\ast||\*x_1)}\eta+32d^2\eta T
	\end{equation}
\end{Lem}
\begin{proof}
	The result follows from \citet{abernethy2008competing} by stopping the derivation on the second inequality below Equation~10.
\end{proof}